\documentclass[11pt]{article}
\usepackage[a4paper,top=3cm,bottom=2cm,left=3cm,right=3cm,marginparwidth=1.75cm]{geometry}

\usepackage{xcolor}
\definecolor{mydarkblue}{rgb}{0,0.3,0.5}
\usepackage[colorlinks = true,
            linkcolor = mydarkblue,
            urlcolor  = mydarkblue,
            citecolor = mydarkblue,
            anchorcolor = mydarkblue]{hyperref}
\usepackage[round]{natbib}

\usepackage[normalem]{ulem}
\usepackage[none]{hyphenat}
\usepackage{breakcites}
\usepackage[utf8]{inputenc} 
\usepackage[T1]{fontenc}    
\usepackage{hyperref}       
\usepackage{url}            
\usepackage{booktabs}       

\usepackage{colortbl}

\usepackage{amsfonts}       
\usepackage{nicefrac}       
\usepackage{microtype}      
\usepackage{lipsum}
\usepackage[linesnumbered,ruled,vlined]{algorithm2e}
\usepackage{algorithmic}

\SetCommentSty{mycommfont}
\SetKwInput{KwInput}{Input}                
\SetKwInput{KwOutput}{Output}  
\SetKwInput{KwInitialization}{Initialization}
\usepackage{float}
\usepackage{caption}
\usepackage{subcaption}
\usepackage{mathtools}
\usepackage{soul}
\usepackage{enumitem}
\usepackage{amssymb}
\usepackage{amsthm}
\usepackage{tikz}
\usepackage{bm}
\usetikzlibrary{arrows}
\usepackage{dsfont}
\usepackage{graphicx, graphics}
\usepackage{wrapfig}
\usepackage{thmtools,thm-restate}

\usepackage{cleveref}
\usepackage{caption}
\usepackage{subcaption}

\newtheorem{theorem}{Theorem}
\newtheorem{lemma}{Lemma}

\newtheorem{constraint}{Constraint}

\DeclareMathAlphabet{\mathbsf}{OT1}{cmss}{bx}{n}
\DeclareMathAlphabet{\mathssf}{OT1}{cmss}{m}{sl}
\crefname{lemma}{lemma}{lemmas}
\Crefname{lemma}{Lemma}{Lemmas}
\crefname{thm}{theorem}{theorems}
\Crefname{thm}{Theorem}{Theorems}
\crefname{prop}{proposition}{propositions}
\Crefname{prop}{Proposition}{Propositions}
\crefname{defn}{definition}{definitions}
\Crefname{defn}{Definition}{Definitions}
\creflabelformat{equation}{#1#2#3}
\crefname{equation}{equation}{equations}
\Crefname{equation}{Equation}{Equations}
\Crefname{section}{Section}{Sections}
\Crefname{appendix}{Appendix}{Appendices}
\crefname{figure}{figure}{figures}
\Crefname{figure}{Figure}{Figures}
\crefname{algorithm}{algorithm}{algorithms}
\Crefname{algorithm}{Algorithm}{Algorithms}
\crefname{assumption}{assumption}{assumptions}
\Crefname{assumption}{Assumption}{Assumptions}

\DeclareMathOperator*{\argmin}{arg\,min}
\newtheoremstyle{named}{}{}{\itshape}{}{\bfseries}{.}{.5em}{\thmnote{#3's }#1}
\theoremstyle{named}
\newtheorem*{namedtheorem}{Theorem}
\newtheorem*{namedlemma}{Lemma}

\usepackage[defaultlines=4,all]{nowidow}
\setnowidow[4]
\setnoclub[4]

\usepackage{libertine}

\newtheorem*{theorem-non}{Theorem}

\title{\textbf{A Closer Look at In-Context Learning under Distribution Shifts}}

\date{}
\author{
  Kartik Ahuja\footnote{Correspondance to kartikahuja@meta.com} \\
  FAIR (Meta AI)\\
  \and
  David Lopez-Paz\\
 FAIR (Meta AI) \\
}

\usepackage{setspace}
\usepackage{parskip}

\begin{document}
\maketitle

\begin{abstract}
  In-context learning, a capability that enables a model to learn from input examples on the fly without necessitating weight updates, is a defining characteristic of large language models. In this work, we follow the setting proposed in \citep{garg2022can} to better understand the generality and limitations of in-context learning from the lens of the simple yet fundamental task of linear regression. The key question we aim to address is: Are transformers more adept than some natural and simpler architectures at performing in-context learning under varying distribution shifts? To compare transformers, we propose to use a simple architecture based on set-based Multi-Layer Perceptrons (MLPs). We find that both transformers and set-based MLPs exhibit in-context learning under in-distribution evaluations, but transformers more closely emulate the performance of ordinary least squares (OLS). Transformers also display better resilience to mild distribution shifts, where set-based MLPs falter. However, under severe distribution shifts, both models' in-context learning abilities diminish.
\end{abstract}

\section{Introduction} 

Transformers \citep{vaswani2017attention} form the backbone of modern large language models (LLMs) including the likes of GPT-3 \citep{brown2020language} and GPT-4 \citep{openai2023gpt}. These LLMs demonstrate remarkable capabilities, such as in-context learning and natural language-based algorithmic reasoning. However, we are only beginning to understand the origins, limitations, and generality of these capabilities, which is essential for developing safe and reliable LLMs.

In-context learning (ICL) refers to a model's capability to acquire knowledge on the fly from examples provided at test time without requiring any weight updates. This ability is especially useful when the model has to adapt to new tasks from a few demonstrations in the test prompt, for example, adapting a model to drive in a new region with few demonstrations. Understanding ICL for LLMs such as GPT-3 trained on raw text data is particularly challenging.  In \cite{garg2022can}, the authors propose an
insightful  training setup, which abstracts away the raw nature of text data. In their work, transformer models from GPT-2 family are trained on prompts comprising of input, label demonstrations and shown to emulate the ordinary least squares (OLS) algorithm.  Certain natural questions arise at this point. What specifics of the transformer are responsible for the emergence of this behvavior? Can simpler architectures exhibit the same capabilities? How resilient is ICL to distribution shifts?  These are the questions that motivate our work.

To compare with transformers, we propose a natural baseline that is  based on set-based MLPs \cite{zaheer2017deep,lopez2017discovering} that exploit the permutation-invariant nature of the task. Depending on the distribution of test prompts, we categorize in-context learning into in-distribution ICL (ID-ICL) and out-of-distribution ICL  (OOD-ICL). Under ID-ICL, the train distribution of the prompt is identical to the test distribution of the prompt. Under OOD-ICL, the test distribution of prompt sequence is different from the train distribution. When evaluating OOD-ICL, we are particularly interested in the case when the test distribution of prompts is centered on the tail of the train distribution of prompts.  We summarize our key contributions below.

\begin{itemize}
\item First, we derive conditions under which the the optimal model that predicts the label for the current query based on the prompt coincide with the OLS or ridge regression. These are based on known arguments, yet it is important to provide them for completeness. 
\item Despite set-based MLPs being particularly suited for this permutation-invariant task, we find that transformers (GPT-2 family) exhibit better ID-ICL abilities. 
\item Under mild distribution shifts, we find that transformers degrade more gracefully than set-based MLPs. Under more severe distribution shifts, both transformers and set-based MLPs do not exhibit ICL abilities.
\item  ID-ICL performance is not predictive of OOD-ICL performance for both architecture choices. 
\end{itemize}

Moving forward, several questions need to be answered. Why are transformers better than set-based MLPs at ICL?  How can we improve the OOD-ICL abilities of these architectures?

\section{Related Works}

Recent studies have offered intriguing insights into in-context learning (ICL). \citet{olsson2022context} propose that the formation of ``induction heads'', which allow models to copy in-context information, is key to ICL. Building on \cite{garg2022can}'s work, several researchers \cite{akyurek2022learning, von2022transformers, dai2022can} demonstrated that transformer model's ability to implicitly execute gradient descent steps during inference could also be central to ICL, supporting their claims with empirical evidence. \citet{litransformers} explore this setup further by analyzing generalization bounds for the learnability of algorithms. Lastly, \citet{xie2021explanation} focus on data sampled from hidden Markov model and interpret in-context learning through the lens of implicit Bayesian inference. They go on to provide conditions under which models can perform ICL even when prompts have low probability under the training distribution.

\citet{chan2022transformers} studied the impact of inductive bias of pretraining the model on ICL. The authors showed that pretrained transformers exhibit rule-based generalization, while those trained from scratch use exemplar-based generalization, i.e., leverage information from the examples provided in-context to carry out ICL. \citet{kirsch2022general} find that among factors determining the inductive bias of the model, state-size is a more crucial parameter than the model size for ICL abilities. More recently,  \citet{wei2023larger} showed that model size can be a crucial parameter as well. In particular, they show that sufficiently large models such as PaLM-540B  are capable of overriding semantic priors if needed, while smaller counterparts are unable to do so.

\section{In-context Learning under Distribution Shifts}

We start with some standard notation. Inputs and labels are denoted as $x \in \mathbb{R}^{d}$ and  $y\in \mathbb{R}$ respectively. Each prompt $p$ is a sequence of independent and identically distributed (i.i.d.) input, label pairs, denoted as $p=\{(x_i,y_i)\}_{i=1}^{k}$. Each prompt $p$ is sampled independently as follows

\begin{equation}
\begin{split}
 &   f \sim \mathbb{P}_{f}, \\
 &   x_i \sim \mathbb{P}_{x}, \varepsilon_i \in \mathbb{P}_{\varepsilon}, \forall i \in \{1, \cdots, k\},\\
 &   y_i \leftarrow  f(x_i) + \varepsilon_i, \forall i \in \{1, \cdots, k\},\\
\end{split}
\label{eqn1}
\end{equation}

where the labeling function $f$, which is fixed for the entire prompt $p$, is sampled from a distribution $\mathbb{P}_{f}$,  inputs $x_i$ are sampled independently from $\mathbb{P}_x$, $y_i$ is generated by adding some noise $\varepsilon_i$ to the labeling function's output $f(x_i)$. For the prompt $p$, we define its prefix as $p_{j} = ((x_1,y_1),(x_2,y_2), \cdots, x_j)$, where $j \in\{1, \cdots, k\}$. Define the support of prefix $p_j$ as $\mathcal{P}_j$. 
 
Define the risk for model $M$ as $R(M) = \sum_{j=1}^{k}\mathbb{E}\big[\ell\big(M(p_{j}), y_{j}\big)\big]$, where $\ell$ is the loss, $M(p_{j})$ looks at the prefixes $p_j$ and makes the prediction, the loss  is computed w.r.t the true label $y_j$, $\mathbb{E}[\cdot]$ is the expectation over the joint distribution of $(p_j,y_j)$. We want to find a model that minimizes the risk $R(M)$
 i.e., 

\begin{equation}M^{*} \in \argmin_{M} R(M) \label{eqn2} \end{equation}

For the results to follow, we make some standard regularity assumptions that we state as follows. The probability measure associated with  $p_j$ is absolutely continuous w.r.t Lebesgue measure.  The conditional expectation and variance exists, i.e., $|\mathbb{E}[y_j|p_j]|<\infty$ and $\mathsf{Var}[y_j|p_j]<\infty$ for all $p_j \in \mathcal{P}_j$.

\begin{lemma} 
\label{lemma1}
If $\ell$ is the square loss, then the solution to equation \eqref{eqn2} satisfies, $M^{*}(p_j) = \mathbb{E}[y_j|p_j], \text{almost everywhere in } \mathcal{P}_j$, $\forall j \in \{1, \cdots, k\}$. 
\end{lemma}

While the above lemma is stated for square loss, an equivalent statement holds for cross-entropy loss. We now turn to our case study, i.e., linear labeling functions $f$. Each prompt $p$ is sampled as follows

\begin{equation}
\begin{split}
&    \beta \sim \mathcal{N}(0, \Sigma), \text{where}\; \Sigma \in \mathbb{R}^{d\times d}\; \text{is invertible}  \\ 
& x_i \sim \mathbb{P}_{x}, \varepsilon_i \sim \mathcal{N}(0,\sigma^2),  \forall i \in \{1, \cdots, k\} \\ 
&    y_i \leftarrow \beta^{\top}x_i + \varepsilon_i,   \forall i \in \{1, \cdots, k\} 
\end{split}
\label{eqn:lin_prompt}
\end{equation}
where $\beta$ is drawn from a normal distribution with mean zero and covariance $\Sigma$ and noise $\varepsilon_i$ is sampled from a normal distribution with mean zero and variance $\sigma^2$. We break down prefix $p_j$ into a matrix $\boldsymbol{X}_j \in \mathbb{R}^{(j-1)\times d}$ and vector $\boldsymbol{y}_j \in \mathbb{R}^{j-1}$ that stacks the first $j-1$ $x_i$'s and $y_i$'s observed in the prompt up to query $x_j$.  The tuple $(\boldsymbol{X}_j,\boldsymbol{y}_j, x_j)$ captures all the relevant information from $p_j$ for predicting $y_j$. Since $p_1$ has no inputs to look at in the past, we set 
$\boldsymbol{X}_1$, $\boldsymbol{y}_1$ to zero. To better understand the notation, consider the following example, $p = \{(x_1,y_1), (x_2,y_2), (x_3,y_3)\}$. Prefix $p_3= \{(x_1,y_1), (x_2,y_2), x_3\}$, $\boldsymbol{X}_3 = \begin{bmatrix} x_1 \\
x_2 \end{bmatrix} $, $\boldsymbol{y}_3 = \begin{bmatrix} y_1 \\ y_2\end{bmatrix}$. Next, we derive the optimal models $M^{*}(p_j)$ for the data distribution in equation \eqref{eqn:lin_prompt}.  The theorems derived below follows from standard results on linear regression (See \citet{dicker2016ridge, richards2021asymptotics}). We still state and derive these for completeness. 

\begin{theorem}
\label{thm1}
If $\ell$ is the square loss and prompt generation follows equation \eqref{eqn:lin_prompt}, then the optimal model from equation \eqref{eqn2} satisfies, 
$$M^{*}(p_j) = x_j^{\top}(\boldsymbol{X}_j^{\top}\boldsymbol{X}_j + \sigma^2\Sigma^{-1} )^{-1}\boldsymbol{X}_{j}^{\top} \boldsymbol{y}_j$$
almost everywhere in $\mathcal{P}_j, \forall j \in \{1, \cdots, k\}$. 
\end{theorem}

If $\Sigma$ is identity, then the above solution coincides with ridge regression \citep{hoerl1970ridge} using $\sigma^2$ as the ridge penalty. We now study the noiseless setting. To analyze the noiseless case, we will look at the ridge solutions in the limit of $\sigma$ going to zero. 

\begin{theorem}
\label{thm2}
If $\ell$ is the  square loss and prompt generation follows equation \eqref{eqn:lin_prompt} with $\Sigma$ as identity, \footnote{If $\Sigma$ is not identity, then the limit may or may not coincide with OLS; see the Appendix for further discussion.} then  in the limit of $\sigma \rightarrow 0 $ the optimal model from equation \eqref{eqn2} satisfies
$$M^{*}(p_j) = x_j^{\top}\boldsymbol{X}_{j}^{+}\boldsymbol{y}_j$$
almost everywhere in $\mathcal{P}_j, \forall j \in \{1, \cdots, k\}$, where $\boldsymbol{X}_{j}^{+}$ is the Moore-Penrose pseudo-inverse of $\boldsymbol{X}_j$.
\end{theorem}

In the above results (Lemma \ref{lemma1}, Theorem \ref{thm1}, and Theorem \ref{thm2}) we do not use the fact that inputs $x_i$'s are drawn independently.  In Theorem \ref{thm1}, and Theorem \ref{thm2}, we assumed that $\beta$ is drawn from a normal distribution.  For distributions beyond normal, we now argue that if we restrict the search space of models, then the same results continue to hold.

\begin{constraint}
\label{assm: model}
    $M(p_j) = x_j^{\top}m(\boldsymbol{X}_j)\boldsymbol{y}_j$.
\end{constraint}

The above  constraint restricts the model to be linear in test query and also to be linear in the label seen up to that point. We do not impose any restrictions on $m(\cdot)$. In the absence of this constraint, the risk $R(M)$ depends on moments beyond the second order moments of the distribution of $\beta$. Thus the optimal model in the absence of this constraint may not coincide with OLS or ridge regression.

\begin{theorem}
\label{thm3}
    Suppose $\ell$ is the square loss, $\beta$'s and $x_i$'s are drawn from an arbitrary distribution with a finite mean and invertible covariance, rest of the prompt generation follows equation \eqref{eqn:lin_prompt}. In this setting, the solution to equation \eqref{eqn2} under Constraint \ref{assm: model} satisfies $$M^{*}(p_j) = x_j^{\top}(\boldsymbol{X}_j^{\top}\boldsymbol{X}_j + \sigma^2\Sigma^{-1} )^{-1}\boldsymbol{X}_{j}^{\top} \boldsymbol{y}_j$$
almost everywhere in $\mathcal{P}_j$, $\forall j \in \{1, \cdots, k\}$. 
\end{theorem}

So far, we have characterized different conditions under which the optimal model emulates the OLS or the ridge regression on the support of training distribution of the prompts. The study by \citet{garg2022can} demonstrated that transformers, when trained with sufficient data, can emulate OLS regression. Theorem \ref{thm1}, \ref{thm2} suggest that sufficiently high capacity models (that can handle input data of varying lengths)  trained on sufficient amount of data should behave as well as transformers on the prompts sampled from the same distribution as the train distribution. We test this hypothesis in the experiments section.  Outside the support of the training distribution of prompts, performance is not guaranteed to be good, and it depends on the inductive biases -- architecture,  optimizer, and the loss function. Our experiments will examine the bias from the architecture.
We now propose a natural architecture for the task in question.

\paragraph{A natural baseline for the above task}
We revisit the data generation in equation \eqref{eqn1} and parametrize the labeling function. Say the labeling process now is $y_i \leftarrow f(x_i, \beta) + \varepsilon_i$, where $\beta$ is sampled from some distribution. $\mathbb{E}[y_i|x_i, \beta] = f(x_i, \beta)$. Our model will first estimate $\beta$ from the given set of samples $\boldsymbol{X}_j, \boldsymbol{y}_j$.  The estimation of $\beta$ does not depend on the order of inputs and thus estimation should be invariant w.r.t. to the order of inputs. Further, we want to work with architectures that are capable of handling inputs of variable length. For this purpose, the most natural architecture are the ones that accept sets as inputs. We revisit the Theorem 2 in \cite{zaheer2017deep}. The theorem states 

\begin{namedtheorem}{\cite{zaheer2017deep}}
    A function operating on a set $\mathcal{A}$ having elements from a countable universe is a
valid set function iff it can be expressed as $\rho \big(\sum_{a_i\in \mathcal{A}} \phi(a_i)\big)$.
\end{namedtheorem}

The aforementioned theorem is stated for elements from a countable universe, with its extension to uncountable sets provided in \cite{zaheer2017deep}, albeit for fixed-length sets. Since functions of the form $\rho \big(\sum_{a_i\in \mathcal{A}} \phi(a_i)\big)$ are uninversal representers of set-based functions we use them as the basis for our architecture. We pick both $\rho$ and $\phi$ as Multilayer Perceptrons (MLPs), and we use these to estimate the parameter $\beta$. The output from these MLPs is then input into another MLP together with the query $x_j$. The final architecture takes the form $\psi\bigg(\rho \big(\frac{1}{j-1}\sum_{i=1}^{j-1} \phi(x_i,y_i)\big), x_j\bigg)$, where 
$(x_i,y_i)$ are input, label pairs seen up to $x_j$. 
To manage sequences of variable length, we incorporate a normalization term $\frac{1}{j-1}$. Consider the noisy label scenario that we studied in Theorem \ref{thm2}, where the optimal model is defined by $x_j^{\top}(\boldsymbol{X}_j^{\top}\boldsymbol{X}_j + \sigma^2\Sigma^{-1} )^{-1}\boldsymbol{X}_{j}^{\top} \boldsymbol{y}_j$. Here, $\rho \big(\frac{1}{j-1}\sum_{i=1}^{j-1} \phi(x_i,y_i)\big)$ aims to output the best estimate for $\beta$, which is $\hat{\beta}(\boldsymbol{X}_j, \boldsymbol{y}_j) = (\boldsymbol{X}_j^{\top}\boldsymbol{X}j + \sigma^2\Sigma^{-1} )^{-1}\boldsymbol{X}_{j}^{\top} \boldsymbol{y}_j$; note how $\hat{\beta}(\boldsymbol{X}_j, \boldsymbol{y}_j)$ is permutation-invariant. As per \citep{zaheer2017deep}, sufficiently expressive $\rho$ and $\phi$ should be capable of expressing $\hat{\beta}(\boldsymbol{X}_j, \boldsymbol{y}_j)$. The final MLP, $\psi$, must approximate a linear map. Next, we delve into the distribution shifts we consider and their underlying rationale.

\paragraph{Distribution shifts for ICL.} 

In both regression and classification problems, the concept of covariate shift \citep{shimodaira2000improving} is well-understood. Covariate shift refers to the situation where the distribution of the input features, denoted as $\mathbb{P}_x$, changes between training and testing phases, but the conditional distribution of the target variable given the features remains invariant. This idea can be applied to the prompts $p$. When the distribution over prompts changes, but the conditional distribution of the target variable (or response) given the prompt remains invariant, this is referred to as ``covariate shift over prompts''. This is a particularly important setting to test, as it helps us understand the model's ability to learn from novel types of prompts or demonstrations at test time. 

Consider two examples that leverage  equation \eqref{eqn:lin_prompt} as the underlying data generation process. Suppose at train time, we generate prompt sequences with inputs $x_i$'s that are mostly positive and then test on prompts comprised of negative inputs. If between train and test we do not alter the label generation process, then this setting qualifies as covariate shift over prompts. On the other hand, consider the setting, where the only difference from train to test is that during label generation at test time is noisy. In this case, the prompt distribution changes but the conditional distribution of the target conditional on the prompt also changes ($\mathbb{E}[y|p]$ at train time is the OLS solution and at test time it is the ridge regression solution). As a result, this type of shift does not qualify as covariate shift over prompts. We want to remark that the difference between two models that perfectly minimize the expected loss in equation \eqref{eqn2} is not apparent under all types of covariate shifts but those that put much more weight on input sequences that are very low probability at train time. This is one aspect in which our choice of distribution shifts differs from \cite{garg2022can}.

\section{Experiments}

In this section, we experiment with the set-based MLPs detailed earlier and transformers from \cite{garg2022can}. We generate data in line with the equation \eqref{eqn:lin_prompt}.  The inputs $x_i's$ at train time are sampled from $\mathcal{N}(0, I_d)$,  where $I_d$ is the $d$ dimensional identity matrix, and at test time they are sampled from  $\mathcal{N}(\mu, I)$. In one case, we set $\mu = 2\cdot \boldsymbol{1}$ and refer to it as a mild distribution shift, and in another case we set $\mu = 4\cdot \boldsymbol{1}$ as severe distribution shift, where $\boldsymbol{1}$ is a $d$ dimensional vector of all ones. The results are presented for $d=10$. The covariance of $\beta$, i.e., $\Sigma$ is identity. We present results for both noiseless labels and noisy labels with $\sigma^2=1$.  For the set-based MLPs, which we refer to as MLP-set, we compare the performance of MLP-set under varying depths, $\{4,5,10,17,26\}$ (indexed from $0$ to $4$ in the increasing order of depth). The width was same for all the layers at $500$. We trained the MLP-set model with the Adam optimizer and a learning rate of $0.001$ except for the case of depth $26$, where we had to lower the learning rate to $0.0001$ to enable learning. We used ReLU activations and batch norm between any two hidden layers. For training the transformer model, we adopt the same architecture used in \citep{garg2022can}, which belongs to the GPT-2 family, and we include performances at two depths - $12$ (Transformer $1$) and $16$ (Transformer $2$).

\begin{figure}[h!]
\begin{minipage}{0.32\textwidth}
    \centering
    \includegraphics[width=2in]{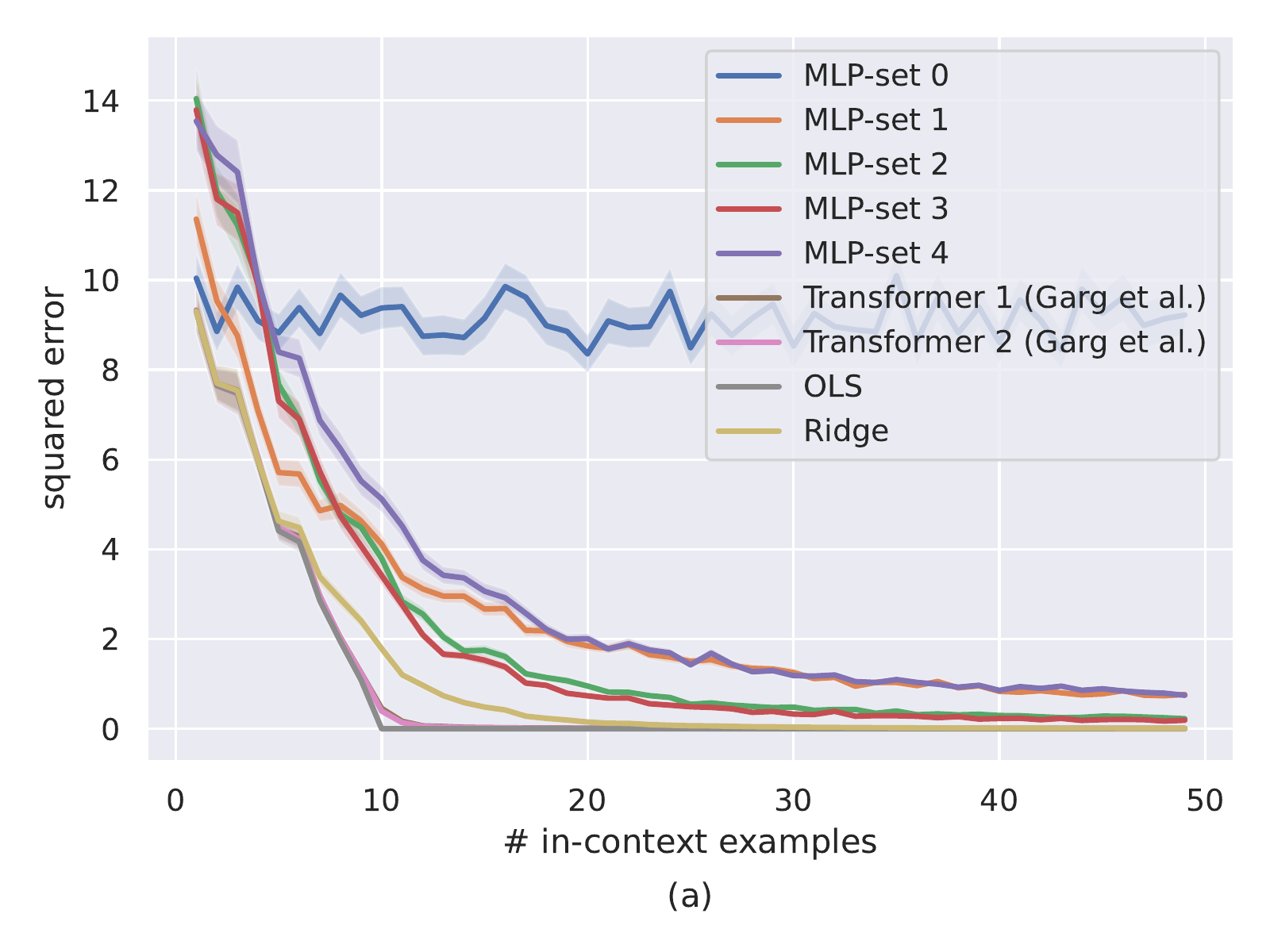}
\end{minipage}
\begin{minipage}{0.32\textwidth}
    \centering
    \includegraphics[width=2in]{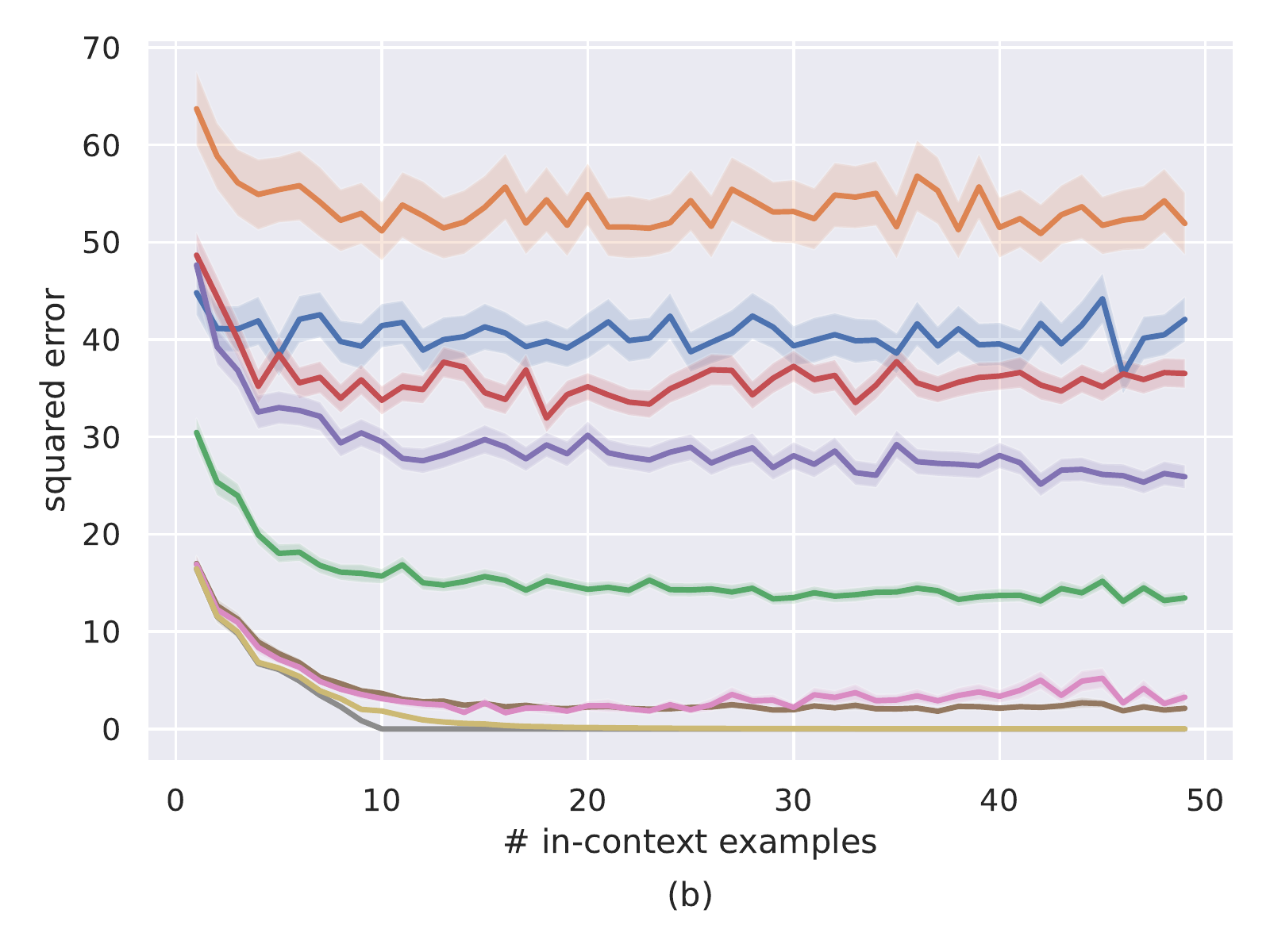}
\end{minipage}
\begin{minipage}{0.32\textwidth}
    \centering
    \includegraphics[width=2in]{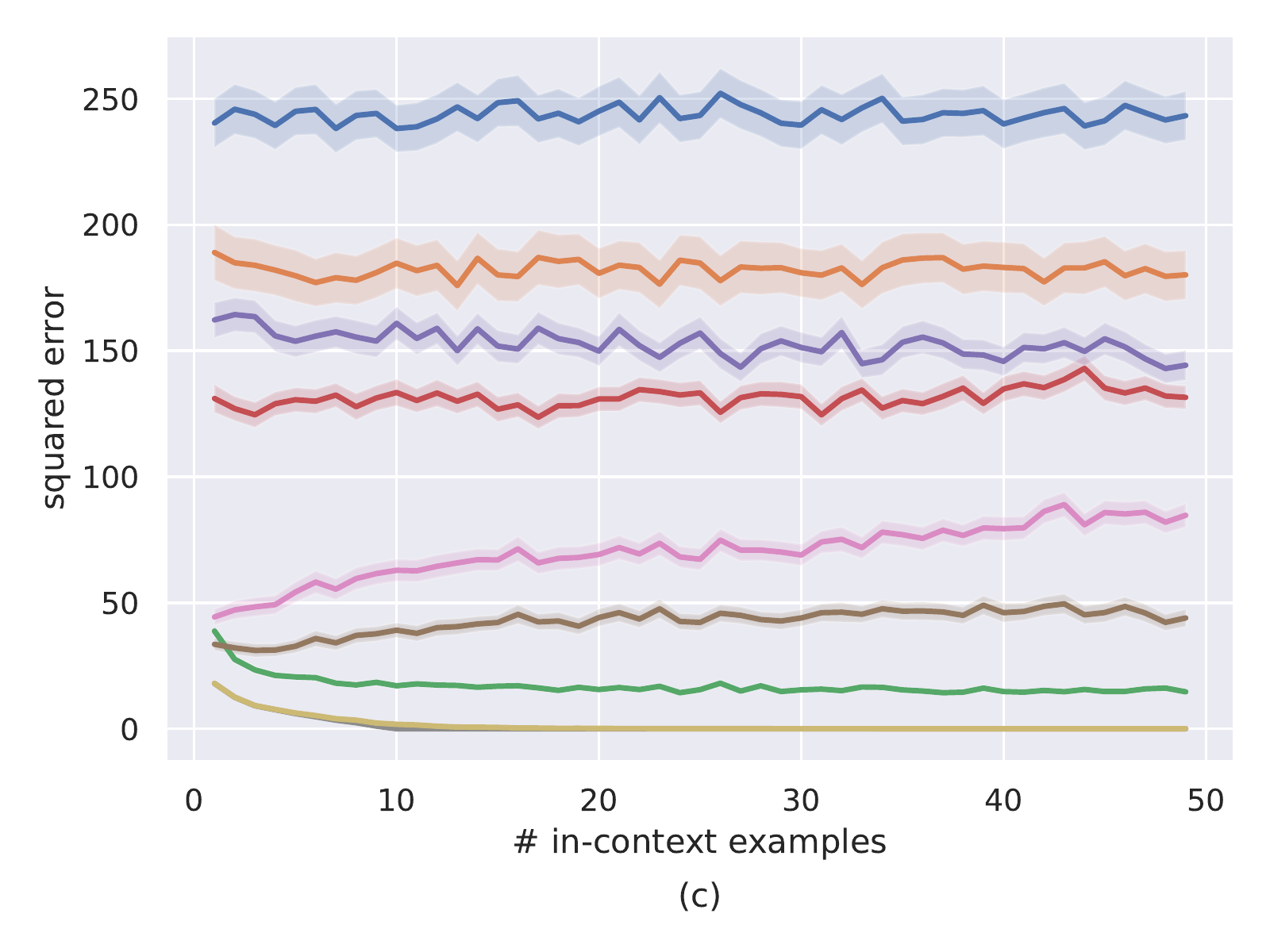}
\end{minipage}
\caption{Comparison of MLP-set and transformers for noiseless setting, i.e., $\sigma=0$. a) ID-ICL ($\mu=0$), b) OOD-ICL (Mild distribution shift with $\mu = 2 \cdot \boldsymbol{1}$), c) OOD-ICL (Severe distribution shift with $\mu = 4\cdot \boldsymbol{1}$).}
\label{figure2}
\end{figure}

\begin{figure}[h!]
\begin{minipage}{0.32\textwidth}
    \centering
    \includegraphics[width=2in]{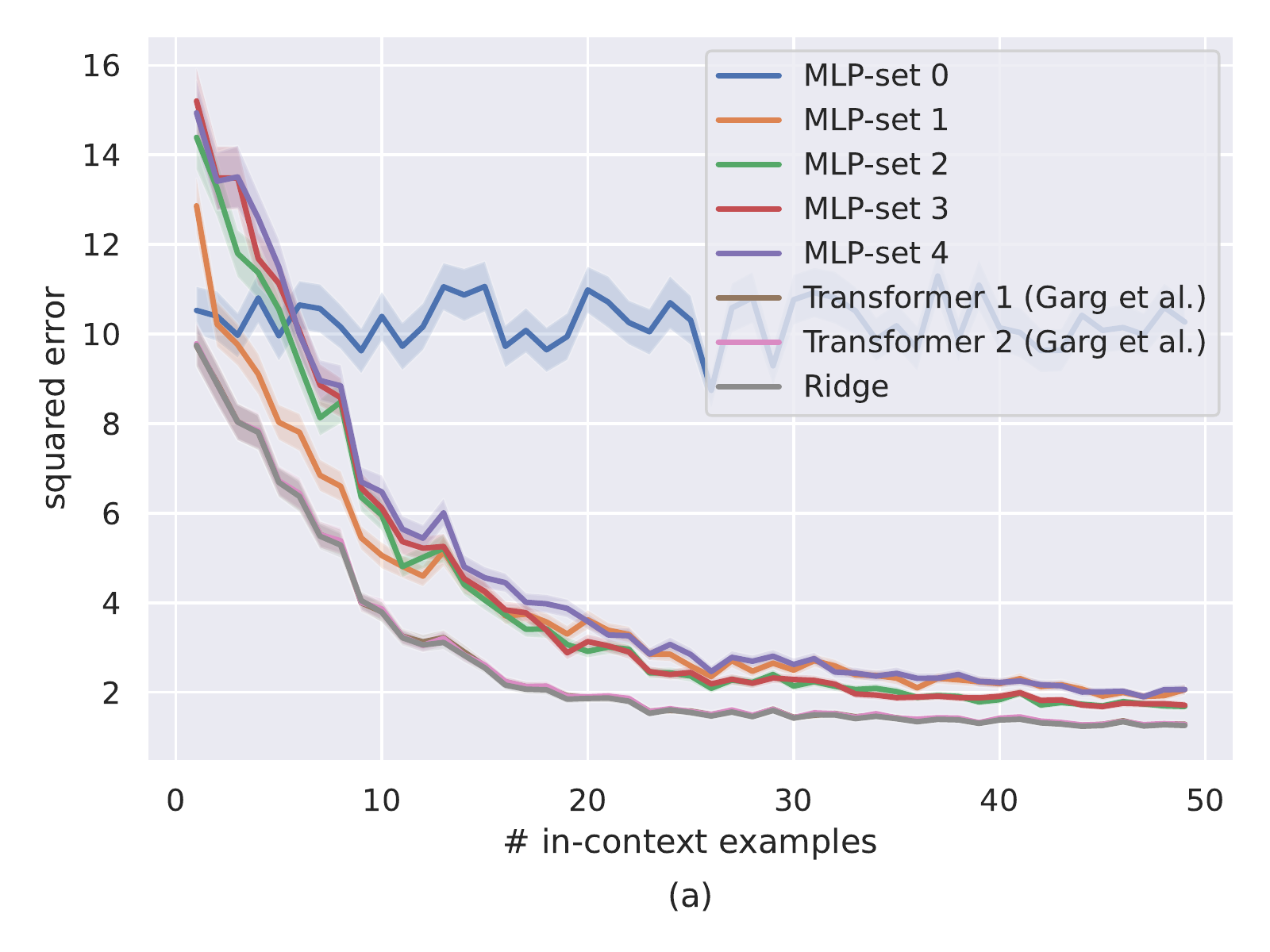}
\end{minipage}
\begin{minipage}{0.32\textwidth}
    \centering
    \includegraphics[width=2in]{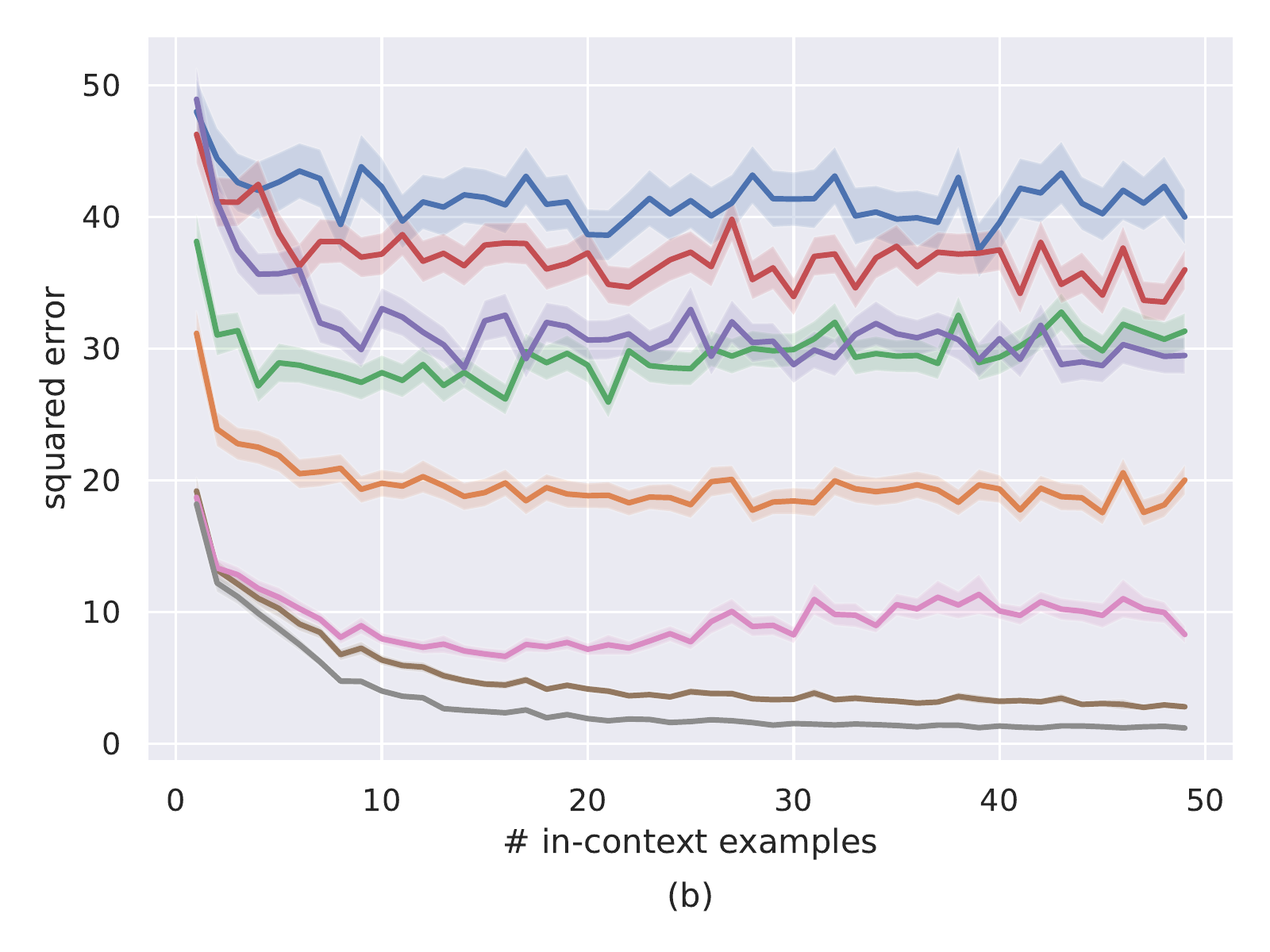}
\end{minipage}
\begin{minipage}{0.32\textwidth}
    \centering
    \includegraphics[width=2in]{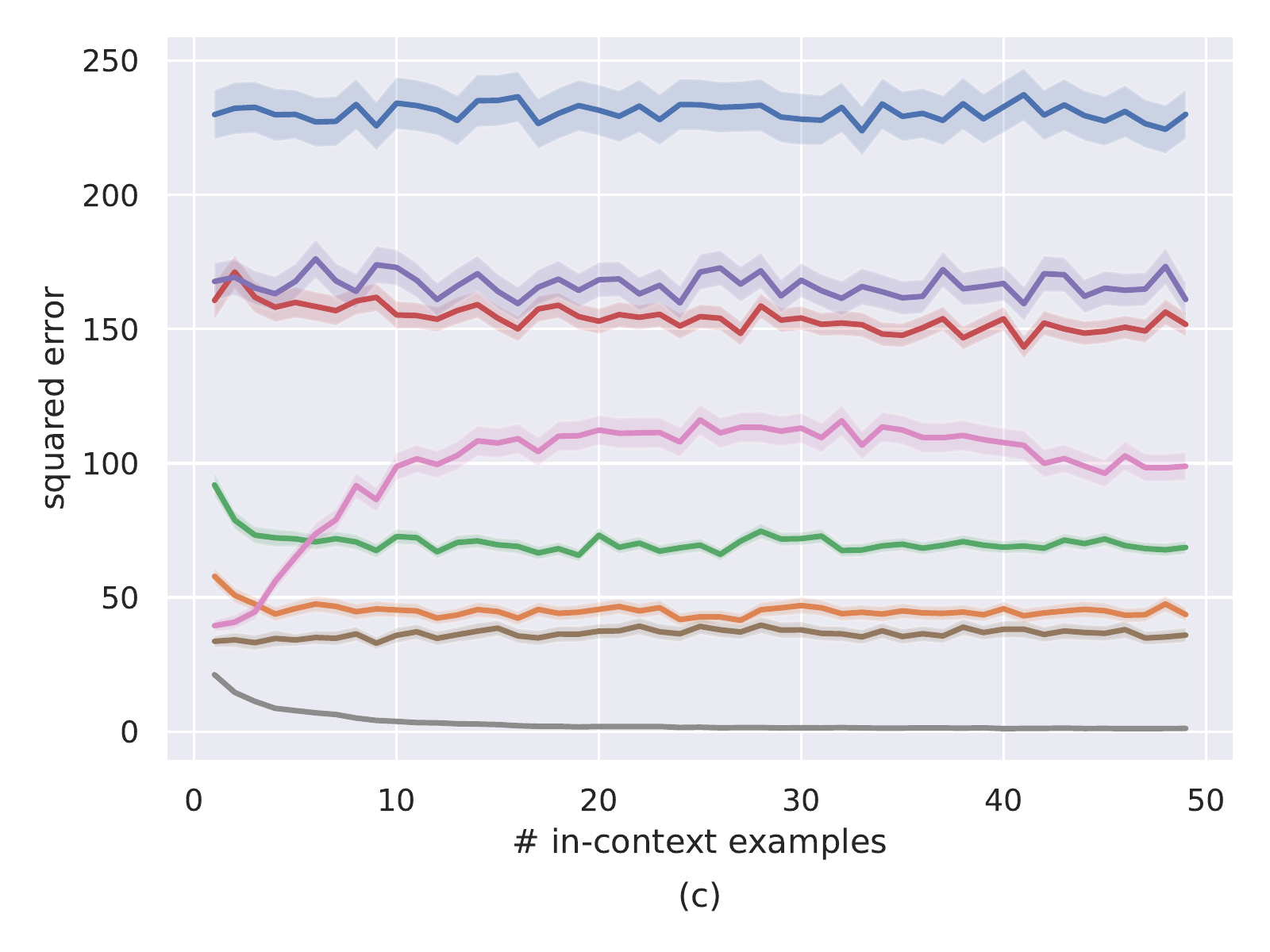}
\end{minipage}
\caption{Comparison of MLP-set and transformers for noisy setting, i.e., $\sigma=1$. a) ID-ICL ($\mu=0$), b) OOD-ICL (Mild distribution shift with $\mu = 2 \cdot \boldsymbol{1}$), c) OOD-ICL (Severe distribution shift with $\mu = 4\cdot \boldsymbol{1}$).}
\label{figure3}
\end{figure}

With this experimental setup we ask these key questions: existing works studying this ICL framework from \citep{garg2022can} focused on transformers exhibiting this capabiltiy. Can this ability exist in other models such as the set-based MLPs?  How do the two architectures differ under distribution shifts?  In Figure \ref{figure2}, \ref{figure3}, we compare the two architectures for the noiseless and noisy setting respectively. We describe our key findings below

\begin{itemize}
\item We find that set-based MLPs exhibit ID-ICL capabilities but do not match the performance of transformers; see Figure \ref{figure2}a, \ref{figure3}a. This is inspite of choosing an architecture that is well suited for the task.
\item Under mild distribution shifts; see Figure \ref{figure2}b, \ref{figure3}b, transformers exhibit a more graceful degradation as opposed set-based MLPs that become more erratic.
\item Under more severe distribution shifts; see Figure \ref{figure2}c, \ref{figure3}c, both the transformers and the set-based MLPs do not exhibit OOD-ICL abilities. 
\item Finally, the ranking of ID-ICL performance of either the set-based MLPs or the transformers is not predictive of their OOD-ICL abilities. 
\end{itemize}

The code for these experiments can be found at \url{https://github.com/facebookresearch/iclmlp}.

\section{Discussion}

This research reveals that transformers outperform natural baselines in approximating OLS and ridge regression algorithms under mild distribution shifts. The question remains, why are transformers superior? Further investigation is required to theorize why transformers when optimized with familiar optimizers like stochastic gradient descent (SGD), can achieve better approximations of algorithms than set-based MLPs. Additionally, it's crucial to explore if these comparisons hold up for a broader set of algorithms (beyond OLS), architectures (beyond set-based MLPs), and understand why. Some important steps towards these inquiries have been made by \citet{liu2022transformers}.


\bibliographystyle{apalike}

\bibliography{ICL-DS}

\newpage
\appendix
\onecolumn

\section{Appendix}

\begin{namedlemma}{[Restatement of Lemma \ref{lemma1}]}
\label{lemma1_app}
If $\ell$ is the square loss, then the solution to equation \eqref{eqn2} satisfies, $M^{*}(p_j) = \mathbb{E}[y_j|p_j], \text{almost everywhere in } \mathcal{P}_j$, $\forall j \in \{1, \cdots, k\}$. 
\end{namedlemma}

\begin{proof}
We write $$R(M) = \sum_{j=1}^{k}R_j(M),$$ where $R_j(M) = \mathbb{E}\bigg[\ell\big(M(p_{j}), y_{j}\big)\bigg] $.  We simplify $R_j(M)$ below 

\begin{equation}
    \begin{split}  
    & R_j(M)   \\
        &=  \mathbb{E}\bigg[\ell\big(M(p_{j}), y_{j}\big)\bigg] = \mathbb{E}_{p_j} \mathbb{E}_{y_j|p_j}\bigg[\big(M(p_j) - y_j\big)^2 \bigg] \\
& = \mathbb{E}_{p_j} \mathbb{E}_{y_j|p_j}\bigg[\big(M(p_j) - \mathbb{E}[y_j|p_j] + \mathbb{E}[y_j|p_j] - y_j\big)^2 \bigg] \\
& = \mathbb{E}_{p_j} \mathbb{E}_{y_j|p_j}\bigg[\big(M(p_j) - \mathbb{E}[y_j|p_j]\big)^2\bigg] + \mathbb{E}_{p_j} \mathbb{E}_{y_j|p_j}\bigg[\big(y_j - \mathbb{E}[y_j|p_j]\big)^2\bigg] + \\ 
&2 \mathbb{E}_{p_j} \mathbb{E}_{y_j|p_j}\bigg[ \big(M(p_j) - \mathbb{E}[y_j|p_j]\big)\big(y_j - \mathbb{E}[y_j|p_j]\big) \bigg]\\
&= \mathbb{E}_{p_j} \bigg[\big(M(p_j) - \mathbb{E}[y_j|p_j]\big)^2\bigg] + \mathbb{E}_{p_j} \big[\mathsf{Var}[y_j|p_j]\big]
\label{eqn1:lemma1_proof}
    \end{split}     
\end{equation}

Observe that $R_j(M) \geq \mathbb{E}_{p_j} \big[\mathsf{Var}[y_j|p_j]\big] $ and thus  $R(M) \geq \sum_{j=1}^{k}\mathbb{E}_{p_j} \big[\mathsf{Var}[y_j|p_j]\big]$. If $M^{*}$ is a minimizer of $R(M)$, then it also has to minimize $R_j(M)$. If that were not the case, then $M^{*}$ could be strictly improved by replacing $M^{*}$ for the $j^{th}$ query with the better model, thus leading to a contradiction. 
Consider the model $\tilde{M}(p_j) = \mathbb{E}[y_j|p_j]$ for all $p_j \in \mathcal{P}_j, \forall j \in \{1, \cdots, k\}$.  This model $\tilde{M}$ minimizes $R(M)$ and each $R_j(\tilde{M})$. Observe that $R_j(\tilde{M}) = \mathbb{E}_{p_j} \big[\mathsf{Var}[y_j|p_j]\big]$.   Therefore, for any minima $M^{*}$, $R_j(M^{*}) = \mathbb{E}_{p_j} \big[\mathsf{Var}[y_j|p_j]\big]$. From equation \eqref{eqn1:lemma1_proof}, we obtain that $\mathbb{E}_{p_j} \bigg[\big(M^{*}(p_j) - \mathbb{E}[y_j|p_j]\big)^2\bigg]=0$. From Theorem 1.6.6 in \cite{ash2000probability}, it follows that $M^{*}(p_j) = \mathbb{E}[y_j|p_j]$ almost everywhere in $\mathcal{P}_j$. 
\end{proof}

\begin{namedtheorem}{[Restatement of Theorem \ref{thm1}.]}
If $\ell$ is the square loss and prompt generation follows equation \eqref{eqn:lin_prompt}, then the optimal model from equation \eqref{eqn2} satisfies, 
$$M^{*}(p_j) = x_j^{\top}(\boldsymbol{X}_j^{\top}\boldsymbol{X}_j + \sigma^2\Sigma^{-1} )^{-1}\boldsymbol{X}_{j}^{\top} \boldsymbol{y}_j$$
almost everywhere in $\mathcal{P}_j, \forall j \in \{1, \cdots, k\}$. 
\end{namedtheorem}

\begin{proof}
    From Lemma \ref{lemma1},  we know that $M^{*}(p_j) = \mathbb{E}[y_j|p_j]$ almost everywhere in $\mathcal{P}_j$. 
    We now simplify $\mathbb{E}[y_j|p_j]$ for the data generation provided in equation \eqref{eqn:lin_prompt}. We follow standard steps of computing the posterior in Bayesian linear regression to obtain the posterior of $\beta$ conditioned on prefix $p_j$

    \begin{equation}
    \begin{split}
       &\log\big(p(\beta | p_j)\big)= \log p\big(\beta | \boldsymbol{X}_j,\boldsymbol{y}_j, x_j \big) = \log p \big(\beta | \boldsymbol{X}_j,\boldsymbol{y}_j\big) \\&
       = \log \big(p\big(\boldsymbol{X}_j,\boldsymbol{y}_j | \beta\big)\big) + \log(p(\beta)) + c \\ 
       &  = -\frac{1}{2\sigma^2}\|\boldsymbol{X}_j\beta - \boldsymbol{y}_j \|^2 -\frac{1}{2}\beta^{\top}\Sigma^{-1}\beta + c \\
       & =  -\frac{1}{2} (\beta-\mu)^{\top}\tilde{\Sigma}^{-1} (\beta-\mu) + c
    \end{split}
    \end{equation}
    where $\tilde{\mu} = \tilde{\Sigma} \boldsymbol{X}_{j}^{\top} \boldsymbol{y}_j$ and $\tilde{\Sigma}=(\boldsymbol{X}_j^{\top}\boldsymbol{X}_j + \sigma^2\Sigma^{-1})^{-1}$. Therefore, $\beta$ conditioned on  $p_j$ is a Gaussian distribution with mean $\tilde{\mu} $ and covariance $\tilde{\Sigma}$. Recall 
    $$y_j = \beta^{\top}x_j + \varepsilon_j$$
    From the linearity of expectation and the expression above for the posterior, it follows 
    $$\mathbb{E}[y_j|p_j] = \mathbb{E}[y_j |\boldsymbol{X}_j,\boldsymbol{y}_j, x_j]  = \mathbb{E}[\beta^{\top}x_j | \boldsymbol{X}_j,\boldsymbol{y}_j, x_j] = \tilde{\mu}^{\top} x_j$$
This completes the proof.     
\end{proof}

\begin{namedtheorem}{[Restatement of Theorem \ref{thm2}]}
 If $\ell$ is the square loss and prompt generation follows equation \eqref{eqn:lin_prompt} with $\Sigma$ as identity,  then  in the limit of $\sigma \rightarrow 0 $ the optimal model from equation \eqref{eqn2} satisfies
$$M^{*}(p_j) = x_j^{\top}\boldsymbol{X}_{j}^{+}\boldsymbol{y}_j$$
almost everywhere in $\mathcal{P}_j, \forall j \in \{1, \cdots, k\}$, where $\boldsymbol{X}_{j}^{+}$ is the Moore-Penrose pseudo-inverse of $\boldsymbol{X}_j$.
\end{namedtheorem}

\begin{proof}
For clarity, in this case we make the dependence of $M^{*}(p_j)$  on $\sigma$ explicit and instead write it as $M^{*}(p_j,\sigma)$ 
 We calculate the limit of the ridge regression predictor as $\sigma$ goes to zero. We obtain $$\lim_{\sigma\rightarrow 0} M^{*}(p_j,\sigma) = x_j^{\top}\lim_{\sigma \rightarrow 0}(\boldsymbol{X}_j^{\top}\boldsymbol{X}_j + \sigma^2\Sigma^{-1} )^{-1}\boldsymbol{X}_{j}^{\top} \boldsymbol{y}_j = x_j^{\top} \boldsymbol{X}_j^{+} \boldsymbol{y}_j$$ 

 In the simplification above, we used $\Sigma$ is identity and  also used the standard limit definition of Moore-Penrose pseudo-inverse \cite{albert1972regression}. 

\end{proof}

\paragraph{Implications for Theorem \ref{thm2} when $\Sigma$ is not identity}
 Now consider the more general case when $\Sigma$ is not identity. In this case, suppose the inverse of  $\boldsymbol{X}_j^{\top}\boldsymbol{X}_j$ exists, which can happen when the rank of $\boldsymbol{X}_j^{\top}\boldsymbol{X}_j$ is $d$. In this case, $\lim_{\sigma \rightarrow 0}(\boldsymbol{X}_j^{\top}\boldsymbol{X}_j + \sigma^2\Sigma^{-1} )^{-1}\boldsymbol{X}_{j}^{\top}  = \boldsymbol{X}_j^{+}$. To see why this is the case, observe that the map $M^{*}(p_j, \sigma)$ is well defined for all $\sigma$ including that at zero and it is also continuous in $\sigma$. If the inverse of $\boldsymbol{X}_j^{\top}\boldsymbol{X}_j$ does not exist, then the limit may not converge to the  Moore-Penrose pseudo-inverse. Consider the following example.

Let  $\boldsymbol{X}_j = \begin{bmatrix}
1\; 0 \\ 
0\; 0 
\end{bmatrix}$ and $\Sigma^{-1} = \begin{bmatrix}
     a \; b \\ 
     b \; c 
\end{bmatrix}$, where $\Sigma$ is invertible and $c\not=0$. 
\begin{equation}
    \begin{split}
   &    (\boldsymbol{X}_j^{\top}\boldsymbol{X}_j + \sigma^2\Sigma^{-1} )^{-1}\boldsymbol{X}_j^{\top} =  \frac{1}{c + \sigma^2(ac-b^2)}\begin{bmatrix}
  c   \; \;0 \\ 
     - b \; \;0 
\end{bmatrix} \\
   &  \lim_{\sigma \rightarrow 0}  (\boldsymbol{X}_j^{\top}\boldsymbol{X}_j + \sigma^2\Sigma^{-1} )^{-1}\boldsymbol{X}_j^{\top} = \begin{bmatrix}
  1  \;\;  0 \\ 
     - \frac{b}{c}\;\;  0
\end{bmatrix}
    \end{split}
\end{equation}

The $\lim_{\sigma \rightarrow 0}  (\boldsymbol{X}_j^{\top}\boldsymbol{X}_j + \sigma^2\Sigma^{-1} )^{-1}\boldsymbol{X}_j^{\top}  \not= \boldsymbol{X}_j^{+}$.

\begin{namedtheorem}{[Restatement of Theorem \ref{thm3}]}
 Suppose $\ell$ is the square loss, $\beta$'s and $x_i$'s are drawn from an arbitrary distribution with a finite mean and invertible covariance, rest of the prompt generation follows equation \eqref{eqn:lin_prompt}. In this setting, the solution to equation \eqref{eqn2} under Constraint \ref{assm: model} satisfies  $$M^{*}(p_j) = x_j^{\top}(\boldsymbol{X}_j^{\top}\boldsymbol{X}_j + \sigma^2\Sigma^{-1} )^{-1}\boldsymbol{X}_{j}^{\top} \boldsymbol{y}_j$$
almost everywhere in $\mathcal{P}_j$, $\forall j \in \{1, \cdots, k\}$. 
\end{namedtheorem}

\begin{proof}
    Recall that $R(M) = \sum_{j}R_j(M)$, where $ R_j(M) = \mathbb{E}\big[(M(p_j)-y_j)^2\big]  $. Let us simplify one of the terms $R_j(M)$.

    \begin{equation}
    \begin{split}
         R_j(M) &=   \mathbb{E}\bigg[(M(p_j)-y_j)^2\bigg]   \\
        & = \mathbb{E}\bigg[(M(p_j)-y_j)^2\bigg] = \mathbb{E}\bigg[(M(p_j)-\beta^{\top}x_j)^2\bigg] + \sigma^2 \\ 
        & = \mathbb{E}\bigg[\big(m(\boldsymbol{X}_j)\boldsymbol{y}_j-\beta^{\top}x_j)^2\bigg] + \sigma^2 \\ 
    \end{split} 
    \label{eqn1: proof_bnormal}
    \end{equation}
Suppose the covariance of $x_j$ is $\Lambda$. We write $\Lambda^{\frac{1}{2}}$ to denote the symmetric positive definite square root of $\Lambda$ (Such a square root always exists, see Theorem 3 in \footnote{\url{https://www.math.drexel.edu/~foucart/TeachingFiles/F12/M504Lect7.pdf}}).  We use this to simplify the above expression in equation \eqref{eqn1: proof_bnormal} as follows
    
    \begin{equation}
    \begin{split}
        R_j(M) &=  \mathbb{E}\bigg[\big(m(\boldsymbol{X}_j)\boldsymbol{y}_j-\beta^{\top}x_j)^2\bigg] + \sigma^2  \\
        & = \mathbb{E}\bigg[\big\|\Lambda^{\frac{1}{2}}\big(m(\boldsymbol{X}_j)\boldsymbol{y}_j-\beta\big)\big\|^2\bigg] + \sigma^2  \\ 
        & =  \mathbb{E}\bigg[\big\|\Lambda^{\frac{1}{2}}\big(m(\boldsymbol{X}_j)\boldsymbol{y}_j)\big\|^2\bigg] + \mathbb{E}\bigg[\big\|\Lambda^{\frac{1}{2}} \beta \big\|^2\bigg]   -2 \mathbb{E}\bigg[\boldsymbol{y}_j^{\top}m(\boldsymbol{X}_j)^{\top}\Lambda\beta\bigg] + \sigma^2
    \end{split} 
    \end{equation}
    Let us simplfify the first and the third term in the above. 
    \begin{equation}
        \begin{split}
            \mathbb{E}\bigg[\big\|\Lambda^{\frac{1}{2}}\big(m(\boldsymbol{X}_j)\boldsymbol{y}_j)\big\|^2\bigg] &=  \mathbb{E}\bigg[\boldsymbol{y}_j^{\top} m(\boldsymbol{X}_j)^{\top}\Lambda m(\boldsymbol{X}_j) \boldsymbol{y}_j \bigg] \\ 
            & = \mathbb{E}\bigg[\beta^{\top}\boldsymbol{X}_j^{\top} m(\boldsymbol{X}_j)^{\top}\Lambda m(\boldsymbol{X}_j) \boldsymbol{X}_j \beta \bigg] +   \sigma^2 \mathbb{E}[\mathsf{Trace}[ m(\boldsymbol{X}_j)^{\top}\Lambda m(\boldsymbol{X}_j) ]]
            \label{eqn2: proof_bnormal}
        \end{split}
    \end{equation}
    In the last simplification above, we use the fact that $\boldsymbol{y}_j = \boldsymbol{X}_j\beta+ \boldsymbol{\varepsilon}_j$, where $\boldsymbol{X}_j\in \mathbb{R}^{(j-1)\times d}$ stacks first $j-1$ $x_i$'s and $\boldsymbol{\varepsilon}_j \in \mathbb{R}^{j-1}$ stacks first $j-1$ $\varepsilon_i$'s, and that each component of noise is independent and zero mean.

    Define $\Theta^1 = \mathbb{E}[\boldsymbol{X}_j^{\top} m(\boldsymbol{X}_j)^{\top}\Lambda f(\boldsymbol{X}_j) \boldsymbol{X}_j]$ and $\Theta^2 = m(\boldsymbol{X}_j)^{\top}\Lambda m(\boldsymbol{X}_j)$. Since $\boldsymbol{X}_j$ is independent of $\beta$ the above expression simplifies to 
    \begin{equation}
    \begin{split}
      &  \mathbb{E}\big[\beta^{\top}\Theta^1 \beta \big] + \sigma^2  \mathsf{Trace}[\Theta^2]  = \sum_{i,j} \Theta^{1}_{i,j}\Sigma_{i,j}  + \sigma^2  \mathsf{Trace}[\Theta^{2}] 
    \end{split}
    \end{equation}
    Now let us consider the third term  in equation \eqref{eqn2: proof_bnormal}.
    \begin{equation}
    \begin{split}
   & \mathbb{E}\bigg[\boldsymbol{y}_j^{\top}m(\boldsymbol{X}_j)^{\top}\Lambda\beta\bigg]   =  \mathbb{E}\bigg[\beta^{\top}\boldsymbol{X}_j^{\top}m(\boldsymbol{X}_j)^{\top}\Lambda\beta\bigg]
    \end{split}
    \end{equation}

    Define $\Gamma = \boldsymbol{X}_j^{\top}m(\boldsymbol{X}_j)^{\top}\Lambda$. Since $\boldsymbol{X}_j$ is independent of $\beta$ the above expression simplifies to 

\begin{equation}
    \mathbb{E}\bigg[\beta^{\top}\Gamma \beta\bigg] = \sum_{i,j} \Gamma_{i,j}\Sigma_{i,j}
\end{equation}

    From the above simplifications it is clear that the loss depends on prior on $\beta$ through its mean and covariance only. Therefore, if we use a Gaussian prior with same mean and covariance we obtain the same loss. As a result, we can assume that prior is Gaussian with same mean and covariance and leverage the previous result, i.e., Theorem \ref{thm1}. This completes the proof.  
\end{proof}

\end{document}